\documentclass[preprint, 11pt]{elsarticle}
\usepackage{times}
\usepackage{color}
\usepackage{algorithm}
\usepackage{algorithmic}
\usepackage{amssymb}
\usepackage{amsmath}
\usepackage{amsthm}
\usepackage{stmaryrd}
\usepackage{graphicx}
\usepackage{hyperref}

\newtheorem{assumption}{Assumption}
\newtheorem{lemma}{Lemma}
\newtheorem{prop}{Proposition}
\newtheorem{theorem}{Theorem}

\newcommand{\eg}{\emph{e.g.}~}
\newcommand{\ie}{\emph{i.e.}~}

\renewcommand{\d}[1]{{\rm d}{#1}}

\newcommand{\calB}{{\cal B}}
\newcommand{\calC}{{\cal C}}

\newcommand{\calT}{{\cal T}}
\newcommand{\calU}{{\cal U}}

\newcommand{\calW}{{\cal W}}
\newcommand{\calX}{{\cal X}}

\def\INTERPOLATE{\textsc{INTERPOLATE}}
\def\PARENTS{\textsc{PARENTS}}
\def\RR{\mathbb{R}}
\def\SAMPLE{\textsc{SAMPLE}}
\def\STEER{\textsc{STEER}}
\def\Uadm{{\calU_{\textsf{adm}}}}
\def\Xfree{\ensuremath {{\calX}_{\textit{free}}}}

\def\avg{{\textrm{avg}}}
\def\disc{{\textrm{disc}}}
\def\distg{\textrm{dist}_{\gamma}}

\def\d{\,{\rm d}}
\def\gdd{\ddot{\gamma}}
\def\gd{\dot{\gamma}}
\def\g{\gamma}
\def\gint{\g_{\textsf{int}}}
\def\gdint{\gd_{\textsf{int}}}
\def\gddint{\gdd_{\textsf{int}}}
\def\uint{u_{\textsf{int}}}
\def\mdd{\ddot{m}}
\def\ndqd{\norm{\Delta \qd}}
\def\ndq{\norm{\Delta q}}

\def\nqd{\norm{\qd}}

\def\qdd{{\ddot{q}}}
\def\qd{{\dot{q}}}
\def\qgoal{q_{\rm goal}}
\def\qinit{q_{\rm init}}
\def\taumax{\tau_{\textsf{max}}}
\def\xgoal{\ensuremath {x_{\textit{goal}}}}
\def\xinit{\ensuremath {x_{\textit{init}}}}
\def\xparent{\ensuremath {x_{\textit{parent}}}}
\def\xrand{\ensuremath {x_{\textit{rand}}}}
\def\xsteer{\ensuremath {x_{\textit{steer}}}}
\newcommand\norm[1]{\left\| #1 \right\|}
\newcommand\norminfty[1]{\left\| #1 \right\|_{\infty}}
\newcommand{\abs}[1]{\ensuremath {\left| #1 \right|}}

\journal{Robotics and Autonomous Systems}

\begin{document}

\begin{frontmatter}

\title{Completeness of Randomized Kinodynamic Planners with State-based Steering}

\author[ut,ca]{St\'ephane Caron}
\author[ntu]{Quang-Cuong Pham}
\author[ut]{Yoshihiko Nakamura}

\address[ut]{Department of Mechano-Informatics, The University of Tokyo, Japan.}
\address[ntu]{School of Mechanical and Aerospace Engineering, Nanyang Technological University, Singapore.}
\address[ca]{Corresponding author: \texttt{stephane.caron@normalesup.org}}
  
\begin{abstract}
    Probabilistic completeness is an important property in motion planning.
    Although it has been established with clear assumptions for geometric
    planners, the panorama of completeness results for \emph{kinodynamic}
    planners is still incomplete, as most existing proofs rely on strong
    assumptions that are difficult, if not impossible, to verify on practical
    systems. In this paper, we focus on an important class of kinodynamic
    planners, namely those that interpolate trajectories in the state space. We
    provide a proof of probabilistic completeness for these planners under
    assumptions that can be readily verified from the system's equations of
    motion and the user-defined interpolation function. Our proof relies
    crucially on a property of interpolated trajectories, termed
    \emph{second-order continuity} (SOC), which we show is tightly related to
    the ability of a planner to benefit from denser sampling. We analyze the
    impact of this property in simulations on a low-torque pendulum. Our
    results show that a simple RRT using a second-order continuous interpolation
    swiftly finds solution, while it is impossible for the same planner using
    standard Bezier curves (which are not SOC) to find any solution.\footnote{
        This paper is a revised and expanded version of~\cite{caron2014icra},
        which was presented at the \emph{International Conference on Robotics and
        Automation}, 2014.
    }
\end{abstract}


\end{frontmatter}

\section{Introduction}

A deterministic motion planner is said to be \emph{complete} if it returns
a solution whenever one exists~\cite{latombe1991book}. A \emph{randomized}
planner is said to be \emph{probabilistically complete} if the probability of
returning a solution, when there is one, tends to one as execution time goes to
infinity~\cite{Lav06book}. Theoretical as they may seem, these two notions are
of notable practical interest, as proving completeness requires one to
formalize the problem by hypotheses on the robot, the environment, etc. While
experiments can show that a planner works for a given robot, in a given
environment, for a given query, etc., a proof of completeness is a certificate
that the planner works for a precise \emph{set} of problems. The size of this
set depends on how strong the assumptions required to make the proof are: the
weaker the assumptions, the larger the set of solvable problems.

Probabilistic completeness has been established for systems with
\emph{geometric} constraints~\cite{kavraki1996probabilistic, Lav06book} such as
\eg obstacle avoidance~\cite{hsu2002randomized}. However, proofs for systems
with \emph{kinodynamic} constraints~\cite{lavalle2001randomized,
karaman2011sampling, hsu1997path} have yet to reach the same level of
generality. Proofs available in the literature often rely on strong assumptions
that are difficult to verify on practical systems (as a matter of fact, none of
the previously mentioned works verified their hypotheses on non-trivial
systems). In this paper, we establish probabilistic completeness
(Section~\ref{sec:completeness}) for a large class of kinodynamic planners,
namely those that interpolate trajectories in the state space. Unlike previous
works, our assumptions can be readily verified from the system's equations of
motion and the user-defined interpolation function.

The most important of these properties is \emph{second-order continuity} (SOC),
which states that the interpolation function varies smoothly and locally
between states that are close. We evaluate the impact of this property in
simulations (Section~\ref{sec:simulations}) on a low-torque pendulum.
Experiments validate our completeness theorem, and suggest that SOC is an
important design guideline for kinodynamic planners that interpolate in the
state space.

\section{Background}
\label{sec:background}

\subsection{Kinodynamic Constraints}

Motion planning was first concerned only with \emph{geometric} constraints such
as obstacle avoidance or those imposed by the kinematic structures of
manipulators~\cite{lozano1983spatial, kavraki1996probabilistic, hsu1997path,
lavalle2001randomized}. More recently, \emph{kinodynamic} constraints, which
stem from differential equations of dynamic systems, have also been taken into
account~\cite{donald1993kinodynamic, lavalle2001randomized, hsu2002randomized}.

Kinodynamic constraints are more difficult to deal with than geometric
constraints because they cannot in general be expressed using only
\emph{configuration-space variables} -- such as the joint angles of a
manipulator, the position and the orientation of a mobile robot,
etc. Rather, they involve higher-order derivatives such as velocities
and accelerations. There are two types of kinodynamic constraints:
\begin{description}
    \item[Non-holonomic constraints:] non-integrable \emph{equality}
        constraints on higher-order derivatives, such as found in wheeled
        vehicles~\cite{Lau98book}, under-actuated
        manipulators~\cite{bullo2001kinematic} or space robots.

    \item[Hard bounds:] \emph{inequality} constraints on higher-order
        derivatives such as torque bounds for manipulators~\cite{BobX85ijrr},
        support areas~\cite{wieber2002} or wrench cones for humanoid
        stability~\cite{caron2015rss}, etc.
\end{description}
Some authors have considered systems that are subject to both types of
constraints, such as under-actuated manipulators with torque
bounds~\cite{bullo2001kinematic}.

\subsection{Randomized Planners}

Randomized planners such as such as Probabilistic Roadmaps (PRM)
\cite{kavraki1996probabilistic} or Rapidly-exploring Random Trees (RRT)
\cite{lavalle2001randomized} build a roadmap on the state space. Both rely on
repeated random sampling of the free state space, \ie states with non-colliding
configurations and velocities within the system bounds. New states are
connected to the roadmap using a \emph{steering} function, which is a method
used to drive the system from an initial to a goal configuration. The steering
method may be imperfect, \eg it may not reach the goal exactly, not take
environment collisions into account, only apply to states that are sufficiently
close, etc. The objective of the motion planner is to overcome these
limitations, turning a local steering function into a global planning method.

PRM builds a roadmap that is later used to generate motions between many
initial and final states (many-to-many queries). When new samples are drawn,
they are connected to \emph{all} neighboring states in the roadmap using the
steering function, resulting in a connected graph. Meanwhile, RRT focuses on
driving the system from \emph{one} initial state $\xinit$ towards a goal area
(one-to-one queries). It grows a tree by connecting new samples to \emph{one}
neighboring state, usually their closest neighbor.

Both PRM's and RRT's \emph{extension} step are represented by
Algorithm~\ref{algo:planner}, which relies on the following sub-routines (see
Fig.~\ref{fig:extend} for an illustration):
\begin{itemize}
    \item $\SAMPLE(S)$: randomly sample an element from a set~$S$;

    \item $\PARENTS(x, V)$: return a set of states in the roadmap $V$ from
        which steering towards $x$ will be attempted;

    \item $\STEER(x, x')$: generate a system trajectory from $x$ towards $x'$.
        If successful, return a new node $\xsteer$ ready to be added to the
        roadmap. Depending on the planner, the successfulness criterion may be
        ``reach $x'$ exactly'' or ``reach a vicinity of $x'$''.
\end{itemize}

\begin{algorithm}
    \caption{Extension step of randomized planners (PRM or RRT)}
    \label{algo:planner}
    \begin{algorithmic}[1]
        \REQUIRE initial node $\xinit$, number of iterations $N$
        \STATE $(V, E) \leftarrow (\{\xinit\}, \emptyset)$
        \FOR{$N$ steps}
            \STATE $\xrand \leftarrow \SAMPLE(\Xfree)$ 
            \STATE $X_\textit{parents} \leftarrow \PARENTS(\xrand, V)$
            \FOR{$\xparent$ in $X_\textit{parents}$}
                \STATE $\xsteer \leftarrow \STEER(\xparent,\,\xrand)$
                \IF{$\xsteer$ is a valid state}
                    \STATE $V \leftarrow V \cup \{ \xsteer \}$
                    \STATE $E \leftarrow E \cup \{ (\xparent, \xsteer) \}$
                \ENDIF
            \ENDFOR
        \ENDFOR
        \RETURN $(V, E)$
    \end{algorithmic}
\end{algorithm}

\begin{figure}[t]
  \centering
  \includegraphics[height=5cm]{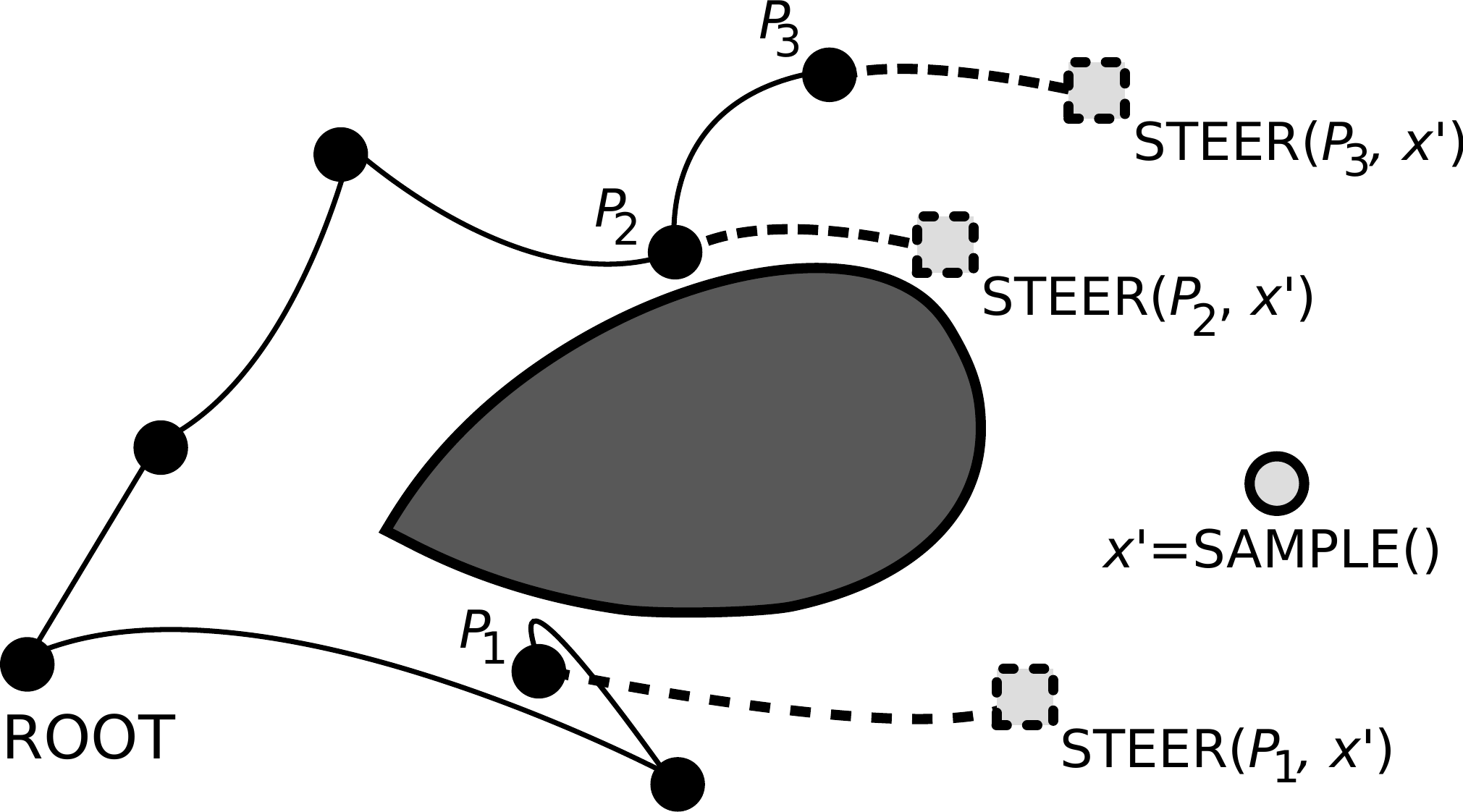}
  \caption{
      Illustration of the extension routine of randomized planners. To grow the
      roadmap toward the sample $x'$, the planner selects a number of parents
      $\PARENTS(x') = \{P_1, P_2, P_3\}$ from which it applies the $\STEER(P_i,
      x')$ method.
  }
  \label{fig:extend}
\end{figure}

The design of each sub-routine greatly impacts the quality and even the
completeness of the resulting planner. In the literature, $\SAMPLE(S)$ is
usually implemented as uniform random sampling over $S$, but some authors have
suggested adaptive sampling as a way to improve planner
performance~\cite{bialkowski2013iros}. In geometric planners, $\PARENTS(q, V)$
is usually implemented from the Euclidean norm over $\calC$ as 
\begin{equation*}
    \PARENTS(q, V) := \underset{q' \in V}{\arg\min} \ \| q' - q \|. 
\end{equation*}
This choice results in the so-called Voronoi bias of
RRTs~\cite{lavalle2001randomized}. Both experiments and theoretical analysis
support this choice for geometric planning, however it becomes inefficient for
kinodynamic planning, as was showed by Shkolnik et al.~\cite{shkolnik2009icra}
on systems as simple as the torque-limited pendulum.

\subsection{Steering Methods}
\label{sec:steering}

This paper focuses on steering functions. These can be classified into three
categories: analytical, state-based and control-based steering.

\paragraph{Analytical steering}

This category corresponds to the ideal case when one can compute analytical
trajectories respecting the system's differential constraints, which are
usually called (perfect) \emph{steering functions} in the
literature~\cite{lavalle2001randomized, papadopoulos2014arxiv}. Unfortunately,
it only applies to a handful of systems. Reeds and Shepp curves for cars are
a notorious example of this~\cite{Lau98book}.

\paragraph{Control-based steering}

Generate a control $u: [0, \Delta t] \to \Uadm$, where $\Uadm$ denotes the set
of \emph{admissible} controls, and compute the corresponding trajectory by
\emph{forward dynamics}. This approach has been called \emph{incremental
simulation}~\cite{kunz2014icra}, \emph{control
application}~\cite{lavalle2001randomized} or \emph{control-space
sampling}~\cite{papadopoulos2014arxiv} in the literature. It is widely
applicable, as it only requires forward-dynamic calculations, but usually
results in weak steering functions as the user has no or limited control over
the destination state. In works such as~\cite{lavalle2001randomized,
hsu2002randomized}, random functions $u$ are sampled from a family of
primitives (\eg piecewise-constant functions), a number of them are tried and
only the one bringing the system closest to the target is retained.
Linear-Quadratic Regulation (LQR)~\cite{perez2012lqr, tedrake2009lqr} also
qualifies as control-based steering: in this case, $u$ is computed as the
optimal policy for a linear approximation of the system given a quadratic cost
function.

\paragraph{State-based steering}

Interpolate a trajectory $\gint: [0, \Delta t] \to \calC$, for
instance a Bezier curve matching the initial and target configurations
and velocities, and compute a control that makes the system track that
trajectory. For fully-actuated system, this is typically done using
\emph{inverse dynamics}. An interpolated trajectory is rejected if no
suitable control can be found. Compared to control-based steering,
this approach applies to a more limited range of systems, but delivers
more control over the destination state.
Algorithm~\ref{algo:steering} gives the prototype of state-based
steering functions.

\begin{algorithm}
    \caption{Prototype of state-based steering functions $\STEER(x, x')$}
    \label{algo:steering}
    \begin{algorithmic}[1]
        \STATE $\gint \leftarrow \INTERPOLATE(x, x')$
        \STATE $\uint := \textsc{INVERSE\_DYNAMICS}(\gint(t), \gdint(t), \gddint(t))$
        \IF {$\forall t \in [0, \Delta t], \uint(t) \subset \Uadm$}
            \RETURN the last state of $\gint$
        \ENDIF
        \RETURN failure
    \end{algorithmic}
\end{algorithm}

\subsection{Previous works}

Randomized planners such as RRT and PRM are both simple to
implement\footnote{ For instance, the RRT used in the simulations of
  this paper was implemented in less than a hunder lines of Python
  code.  } yet efficient for geometric planning. The completeness of
these planners has been established for geometric planning in
\cite{lavalle2001randomized, karaman2011sampling, hsu1997path}. In
their proof, Hsu et al.~\cite{hsu1997path} quantified the problem of
narrow passages in configuration space with the notion of $(\alpha,
\beta)$-expansiveness. The two constants $\alpha$ and $\beta$ express
a geometric lower bound on the rate of expansion of reachability
areas.

There is, however, a gap between geometric and \emph{kinodynamic}
planning~\cite{donald1993kinodynamic} in terms of proving probabilistic
completeness. When Hsu et al. extended their solution to kinodynamic
planning~\cite{hsu2002randomized}, they applied the same notion of
expansiveness, but this time in the $\calX \times \calT$ (state and time) space
with control-based steering. Their proof states that, when $\alpha > 0$ and
$\beta > 0$, their planner is probabilistically complete. However, whether
$\alpha > 0$ or $\alpha = 0$ in the non-geometric space $\calX \times \calT$
remains an open question. As a matter of fact, the problem of evaluating
$(\alpha, \beta)$ has been deemed as difficult as the initial planning
problem~\cite{hsu1997path}. In a parallel line of work, LaValle et
al.~\cite{lavalle2001randomized} provided a completeness argument for
kinodynamic planning, based on the hypothesis of an \emph{attraction sequence},
\ie a covering of the state space where two major problems of kinodynamic
planning are already solved: steering and antecedent selection. Unfortunately,
the existence of such a sequence was not established. 

In the two previous examples, completeness is established under assumptions
whose verification is at least as difficult as the motion planning problem
itself. Arguably, too much of the complexity of kinodynamic planning has been
abstracted into hypotheses, and these results are not strong enough to hold the
claim that their planners are probabilistically complete in general. This was
exemplified recently when Kunz and Stilman~\cite{kunz2015afr} showed that RRTs
with control-based steering were \emph{not} probabilistically complete for
a family of control inputs (namely, those with fixed time step and best-input
extension). At the same time, Papadopoulos et al.~\cite{papadopoulos2014arxiv}
established probabilistic completeness for the same planner using a different
family of control inputs (randomly sampled piecewise-constant functions). The
picture of completeness for kinodynamic planners therefore seems to be
a nuanced one.

Karaman et al.~\cite{karaman2011sampling} introduced the RRT* path planner an
extended it to kinodynamic planning with differential constraints
in~\cite{karaman2010optimal}, providing a sketch of proof for the completeness
of their solution. However, they assumed that their planner had access to the
optimal cost metric and optimal local steering, which restricts their analysis
to systems for which these ideal solutions are known. The same authors tackled
the problem from a slightly different perspective in~\cite{karaman2013sampling}
where they supposed that the $\PARENTS$ function had access to $w$-weighted
boxes, an abstraction of the system's local controlability. However, they did
not show how these boxes can be computed in practice\footnote{
    The definition of $w$-weighted boxes is quite involved: it depends on the
    joint flow of vector fields spanning the tangent space of the system's
    manifold.
} and did not prove their theorem, arguing that the reasoning was similar to
the one in~\cite{karaman2011sampling} for kinematic systems.

To the best of our knowledge, the present paper is the first to
provide a proof of probabilistic completeness for kinodynamic planners
using state-based steering.

\subsection{Terminology}

A function is \emph{smooth} when all its derivatives exist and are continuous.
Let \mbox{$\|\cdot\|$} denote the Euclidean norm. A function $f : A \to B$
between metric spaces is Lipschitz when there exists a constant $K_f$ such that
\[
\forall (x, y) \in A, \ \| f(x) - f(y) \| \,\leq\, K_f \| x - y \|.
\]
The (smallest) constant $K_f$ is called the Lipschitz constant of the function
$f$.

Let $\calC$ denote $n$-dimensional configuration space, where $n$ is the number
of degrees of freedom of the robot. The \emph{state space} $\calX$ is the
$2n$-dimensional manifold of configuration and velocity coordinates $x = (q,
\qd)$. A trajectory is a continuous function $\gamma : [0, \Delta t] \to
\calC$, and the distance of a state $x \in \calX$ to a trajectory $\gamma$ is
\begin{equation*}
    \distg(x) := \min_{t \in [0, \Delta t]} \norm{(\g, \gd)(t) - x}.
\end{equation*}

A kinodynamic system can be written as a time-invariant differential system:
    \begin{equation}
    \label{eq:dynamics}
    \dot{x}(t) \ = \ f(x(t), u(t)),
    \end{equation}
where $u \in \calU$ denotes the control input and $x(t) \in \calX$. Let $\Uadm
\subset \calU$ denote the subset of admissible controls. (For instance, $\Uadm
= [\tau_{\min}, \tau_{\max}] \subset \calU = \mathbb{R}$ represents bounded
torques for a single joint.) A control function $u : [0, \Delta t] \to \calU$
has $\delta$-clearance when its image is in the $\delta$-interior of $\Uadm$,
\ie for any time $t$, $\calB(u(t), \delta) \subset \Uadm$. A trajectory
$\gamma$ that is solution to the differential system \eqref{eq:dynamics} using
only controls $u(t) \in \Uadm$ is called an \emph{admissible} trajectory. The
kinodynamic motion planning problem is to find an admissible trajectory from
$\qinit$ to $\qgoal$.

\section{Completeness Theorem}
\label{sec:completeness}

\subsection{System assumptions}

Our model for an $\calX$-state randomized planner is given by Algorithm
\ref{algo:planner} using state-based steering. We first assume that:

\begin{assumption}
    The system is fully actuated.
\end{assumption}

Full actuation allows us to write the equations of motion of the system in
generalized coordinates as:
    \begin{equation}
    \label{eq:motion}
    M(q) \qdd + C(q, \qd) \qd + g(q) = u,
    \end{equation}
where $u \in \Uadm$ and we assume that the set of admissible controls $\Uadm$
is compact. Since torque constraints are our main concern, we will focus on
\begin{equation}
    \label{Uadm}
    \Uadm := \left\{ u \in \calU, \ | u | \leq \taumax\right\},
\end{equation}
which is indeed compact.\footnote{
    The application of our proof of completeness to an arbitrary compact set
    presents no technical difficulty: one can just replace $|u| \leq \taumax$
    with $d(u, \partial \Uadm)$, with $\partial \Uadm$ the boundary of $\Uadm$.
    Using Equation~\eqref{Uadm} avoids this level of verbosity.
} (Vector comparisons are component-wise.) Finally, we suppose that forward and
inverse dynamics mappings have Lipschitz smoothness:

\begin{assumption}
    The forward dynamics function $f$ is Lipschitz continuous in both of its
    arguments, and its inverse $f^{-1}$ (the inverse dynamics function $u
    = f^{-1}(x, \dot{x})$) is Lipschitz in both of its arguments.
\end{assumption}

These two assumptions are satisfied when $f$ is given by~\eqref{eq:motion} as
long as the matrices $M(q)$ and $C(q, \qd)$ are bounded and the gravity term
$g(q)$ is Lipschitz. Indeed, for a small displacement between $x$ and $x'$,
\begin{equation}
    \label{equ:ub-ctrl}
    \norm{u' - u} 
    \leq 
    \norm{M} \norm{\qdd' - \qdd} 
    + \norm{C(q, \qd)}\norm{\qd' - \qd} 
    + K_g \norm{q' - q}
\end{equation}
Let us illustrate this on the double pendulum depicted in
Figure~\ref{fig:pendulum}. When both links have mass $m$ and length $l$, the gravity term 
\[
    g(\theta_1, \theta_2) = \frac{m g l}{2}\left[ 
    \sin\theta_1 + \sin(\theta_1 + \theta_2) \ \sin(\theta_1 + \theta_2) \right]
\]
is Lipschitz with constant $K_g = 2mgl$, while the inertial term is bounded by
$\norm{M} \leq 3 m l^2$. When joint angular velocities are bounded by $\omega$,
the norm of the Coriolis tensor is bounded by $2 \omega m l^2$.
Using~\eqref{equ:ub-ctrl}, one can therefore derive the Lipschitz constant
$K_{f^{-1}}$ of the inverse dynamics function.

\begin{figure}
  \begin{minipage}[b]{0.45\textwidth}
      \centering
      \includegraphics[height=5cm]{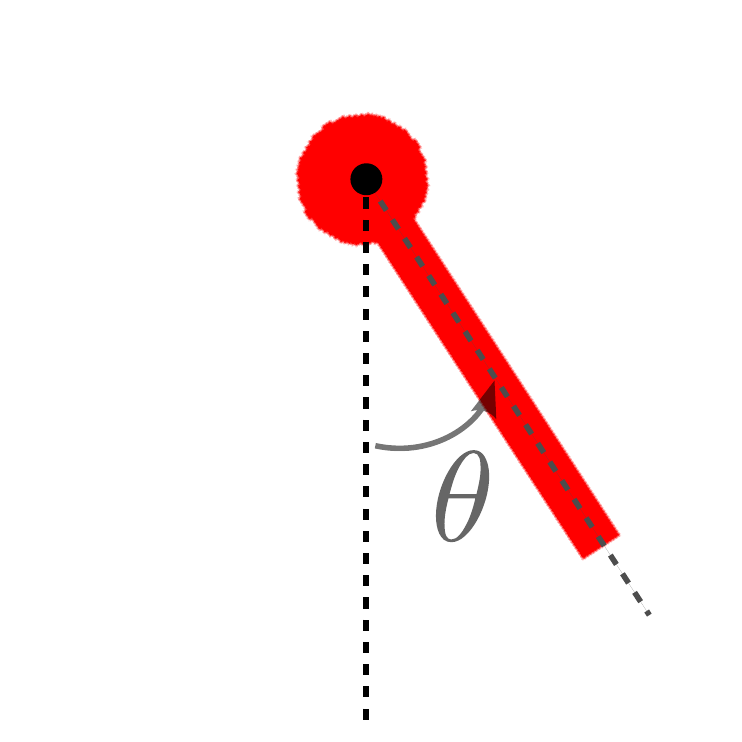}
      \centerline{(A)}\medskip
  \end{minipage}   
  \begin{minipage}[b]{0.55\textwidth}
      \centering
      \includegraphics[height=6cm]{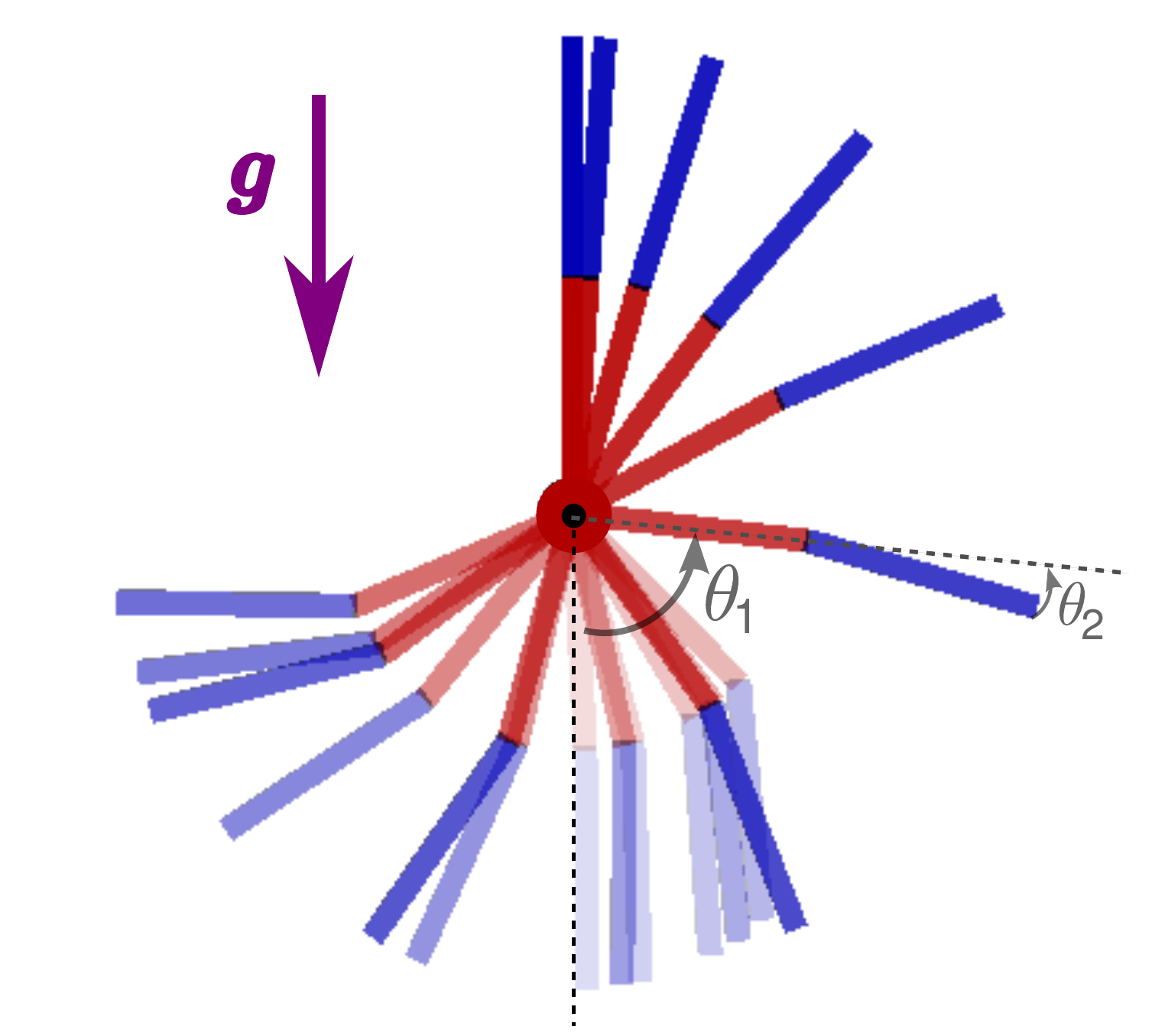}
      \centerline{(B)}\medskip
  \end{minipage}
  \caption{
      Single (A) and double (B) pendulums. Under torque bounds, these systems
      must swing back and forth several times before they can reach for the
      upright position, as depicted in (B) (lighter images represent earlier
      times).
  }
  \label{fig:pendulum}
\end{figure}

\subsection{Interpolation assumptions}

We also require smoothness for the interpolated trajectories:

\begin{assumption}
Interpolated trajectories $\gint$ are smooth Lipschitz functions, and their
time-derivatives $\gdint$ (\ie interpolated velocities) are also Lipschitz.
\end{assumption}

The following two assumptions ensure a continuous behaviour of the
interpolation procedure:

\begin{assumption}[Local boundedness]
\label{ass-int-2}
Interpolated trajectories stay within a neighborhood of their start and end
states, \ie there exists a constant $\eta$ such that, for any $(x, x') \in
\calX^2$, the interpolated trajectory $\gint :~[0, \Delta t] \to \calC$
resulting from $\INTERPOLATE(x, x')$ is included in a ball of center $x$ and
radius $\eta \norm{x' - x}$.
\end{assumption}

\begin{assumption}[Discrete-acceleration convergence]
\label{ass-int-3}
When start and end states become close, accelerations of interpolated
trajectories uniformly converge to the discrete acceleration between them, \ie
there exists some $\nu > 0$ such that, if $\gint : [0, \Delta t] \to \calC$
results from $\INTERPOLATE(x, x')$, then
    \begin{equation*}
    \forall \tau \in [0, \Delta t], \ \norm{\gddint(\tau) - \frac{\Delta \qd}{\Delta t_\disc}}
        \ \leq \ \nu \norm{\Delta x},
    \end{equation*}
where $\Delta t_\disc := {\ndq} / {\nqd}$.
\end{assumption}

\noindent Note that the expression $\frac{\Delta \qd}{\Delta t_\disc}$ above
represents the discrete acceleration between $x$ and $x'$. Its continuous
analog would be $
\frac{\norm{\qd} \d{\qd}}{\norm{\d{q}}} = 
\frac{\norm{\qd} \d{\qd}}{\norm{\qd} \d{t}} = 
\frac{\d{\qd}}{\d{t}}$.

These three assumptions ensure that the planner interpolates trajectories
locally and ``continuously'' when $x$ and $x'$ are close. We will call them
altogether \emph{second-order continuity}, where ``second-order'' refers to the
discrete acceleration encoded in small variations $(\Delta q, \Delta \qd)$.
This continuous behavior plays a key role in our proof of completeness, as it
ensures that denser sampling will allow finding arbitrarily narrow state-space
passages.

Let us consider again the example the double pendulum, for the interpolation
function $\gamma = \INTERPOLATE(x, x')$ given by
    \begin{equation}
    \label{ditch-my-polynomial}
    \begin{array}{rcl}
        \gamma: [0, \Delta t] & \to & \calC \\
        t & \mapsto &  \frac{\Delta \qd}{2 \Delta t} t^2
        + \left(\frac{\Delta q}{\Delta t} - \frac{\Delta \qd}{2} \right) t
        + q.
    \end{array}
    \end{equation}
The duration $\Delta t$ is taken as $\Delta t_{\disc}$, so that $\gamma(0)
= q$, $\gamma(\Delta t) = q'$ and $\ddot{\gamma}$ is the discrete acceleration.
This interpolation, like any polynomial function, is Lipschitz smooth;
Assumption~\ref{ass-int-3} is verified by construction, and
Assumption~\ref{ass-int-2} can be checked as follows:
\begin{eqnarray*}
\norm{\gamma(t) - \gamma(0)} 
    & \leq & t \norm{\frac{1}{2} \frac{\Delta \qd}{\Delta t}
        t + \frac{\Delta q}{\Delta t} - \frac{\Delta \qd}{2}} \\
    & \leq & \Delta t \norm{\frac{\Delta \qd}{2 \Delta t} t + \frac{\Delta
        q}{\Delta t} - \frac{\Delta \qd}{2}} \\
    & \leq & \frac{3}{2} \norm{\Delta \qd} \Delta t + \norm{\Delta q} \\
    & \leq & \norm{\Delta q} \left( 1 + \frac{\norm{\Delta
        \qd}}{\norm{\qd}}\right) \\
    & \leq & \norm{\Delta q} \left(1 + o_{\norm{\Delta x}}(1)\right).
\end{eqnarray*}

\subsection{Completeness theorem}

In order to prove the theorem, we will use the following two lemmas, which are
proved in~\ref{appendix:lemmas}.

\begin{lemma}
    \label{lemma:avg}
    Let $g : [0, \Delta t] \to \RR^k$ denote a smooth Lipschitz function. Then,
    for any $(t, t') \in [0, \Delta t]^2$,
    \begin{equation*}
        \ \left\| \dot{g}(t) - \frac{g(t') - g(t)}{|t' - t|} \right\| 
        \ \leq \ \frac{K_g}{2} | t' - t |.
    \end{equation*}
\end{lemma}

\begin{lemma}
    \label{lemma:minacc}
    If there exists an admissible trajectory $\g$ with
    $\delta$-clearance in control space, then there exists $\delta' <
    \delta$ and a neighboring admissible trajectory $\g'$ with
    $\delta'$-clearance in control space whose acceleration never
    vanishes, \ie such that $\| \gdd' \|$ is always greater than some
    constant $\mdd > 0$.
\end{lemma}

We can now state our main theorem:

\begin{theorem}
\label{th:completeness}
Consider a time-invariant differential system \eqref{eq:dynamics} with
Lipschitz-continuous $f$ and full actuation over a compact set of admissible
controls $\Uadm$. Suppose that the kinodynamic planning problem between two
states $\xinit$ and $\xgoal$ admits a smooth Lipschitz solution $\gamma : [0,
T] \to \calC$ with $\delta$-clearance in control space. A randomized motion
planner (Algorithm \ref{algo:planner}) using a second-order continuous
interpolation is probabilistically complete.
\end{theorem}

\noindent \emph{Proof.} Let $\g : [0, \Delta t] \to \calC, t \mapsto \g(t)$ denote
a smooth Lipschitz \emph{admissible} trajectory from $\xinit$ to $\xgoal$, and
$u : [0, \Delta t] \to \Uadm$ its associated control trajectory with
$\delta$-clearance in control space. Consider two states $x=(q, \qd)$ and $x'
= (q', \qd')$, as well as their corresponding time instants on the trajectory
\begin{eqnarray*}
    t & := & \arg\min_t \norm{(\g(t), \gd(t)) - x}, \\
    t' & := & \arg\min_t \norm{(\g(t), \gd(t)) - x'}.
\end{eqnarray*}
Supposing without loss of generality that $t' > t$, we denote by $\Delta t = t'
- t$ and $\Delta t_\disc = \nqd / \ndq$. Given a sufficiently dense sampling of
the state space, we suppose that $\distg(x) \leq \rho$ and $\distg(x') \leq
\rho$ for a radius $\rho$ such that $\rho / \Delta t = O(\Delta t)$ and
$\rho / \Delta t_\disc = O(\Delta t)$; \ie the radius $\rho$ is quadratic in
the time difference.

Let $\gint : [0, \Delta t] \to \calC$ denote the result of the interpolation
between $x$ and $x'$. For $\tau \in [0, \Delta t]$, the torque required to
follow the trajectory $\gint$ is $\uint(\tau) := f(\gint(\tau), \gdint(\tau),
\gddint(\tau))$. Since $u$ has $\delta$-clearance in control space,
\begin{eqnarray*}
    \abs{\uint(\tau)} 
    & \leq & \abs{\uint(\tau) - u(t)} + \abs{u(t)} \\
    & \leq & \abs{f(\gint(\tau), \gdint(\tau), \gddint(\tau)) - f(\g(t), \gd(t), \gdd(t))} + (1 - \delta)\,\taumax,
\end{eqnarray*}
(As previously, vector inequalities are component-wise.) Let us denote by
$|\widetilde{\uint}|$ the first term of this inequality. We will now show that
$|\widetilde{\uint}| = O(\Delta t) \to 0$ when $\Delta t \to 0$, and therefore that
$|\uint(\tau)| \leq \taumax$ for a small enough $\Delta t$ (\ie when sampling
density is high enough). Let us first rewrite it as follows:
\begin{eqnarray*}
    |\widetilde{\uint}| & = & \abs{f(\gint(\tau), \gdint(\tau), \gddint(\tau))
- f(\g(t), \gd(t), \gdd(t))} \\ 
    & \leq & \norminfty{f(\gint(\tau), \gdint(\tau), \gddint(\tau)) - f(\g(t), \gd(t), \gdd(t))} \\ 
    & \leq & K_f \norm{(\gint(\tau), \gdint(\tau)) - (\g(t), \gd(t))} + K_f \norm{\gddint(\tau) - \gdd(t)}  \\
    & \leq & \underbrace{K_f \left[ (\eta + \nu) \norm{\Delta x} + \distg(x) \right]}_\text{position-velocity term (PV)} + \underbrace{K_f \norm{\frac{\norm{\qd}}{\norm{\Delta q}} \Delta \qd - \gdd(t)}}_\text{acceleration term (A)}.
\end{eqnarray*}
The replacement of the norm $\norm{\cdot}$ by $\norminfty{\cdot}$ is possible
because all norms of $\RR^n$ are equivalent (a change in norm will be reflected
by a different constant $K_f$). The transition from the second to the third row
uses Lipschitz smoothness of $f$, as well as the triangular inequality to
separate position-velocity and acceleration coordinates. The transition from
the third to the fourth row relies on the two interpolation assumptions: local
boundedness (yields the $\eta$ factor in the distance term) and convergence to
the discrete-acceleration (yields the $\nu$ factor in the distance term, as
well as the acceleration term).

The position-velocity term (PV) satisfies:
\begin{equation*}
    (\text{D}) \ \leq \ (2 \rho + \norm{\Delta \g}) (\eta + \nu) + \rho
    \ \leq \ \frac12 K_\g (\eta + \nu) \Delta t + (1 + 2 (\eta + \nu)) \rho.
\end{equation*}
Since $\rho = O(\Delta t)$, we have $(\text{PV}) = O(\Delta t)$ and thus
$|\widetilde{u}| \leq (\text{A}) + O(\Delta t)$. Next, the difference (A) can
be bounded as:
    \begin{eqnarray*}
    (\text{A})
    & \leq & 
    \underbrace{
        \norm{
        \Delta \qd \frac{\norm{\qd}}{\norm{\Delta q}}
        - \Delta \gd \frac{\norm{\gd(t)}}{\norm{\Delta \g}}
        } 
    }_{(\Delta)} +
    \underbrace{
        \frac{\norm{\Delta \gd}}{\norm{\Delta \g}} 
        \left| \norm{\gd(t)} - \frac{\norm{\Delta \g}}{\Delta t} \right|
    }_\text{(A')} \\
        & + & 
    \underbrace{
        \norm{ \frac{\Delta \gd}{\Delta t} - \gdd(t)}.
    }_\text{(A'')}
    \end{eqnarray*}
From Lemma~\ref{lemma:avg}, the two terms (A') and (A'') satisfy:
\begin{eqnarray*}
    (\text{A'}) & \leq & \frac{K_{\gd}}{2} \frac{\norm{\Delta \gd}}{\norm{\Delta \g}} \Delta t \ = \ O(\Delta t), \\
    (\text{A''}) & \leq & \frac{K_{\gd}}{2} \Delta t \ = \ O(\Delta t),
\end{eqnarray*}
where the first upper bound $O(\Delta t)$ comes from the fact that
$\frac{\norm{\Delta \gd}}{\norm{\Delta \g}} \underset{\Delta t \to 0}{\sim}
\Delta t$. We now have $|\widetilde{u}| \leq (\Delta) + O(\Delta t)$. The term
($\Delta$) can be seen as the deviation between the discrete accelerations of
$\gint$ and $\g$. Let us decompose it in terms of norm and angular deviations:
\begin{eqnarray*}
    (\Delta) 
    & \leq & \left\| \left( \frac{\Delta \gd}{\norm{ \Delta \gd}} - \frac{\Delta \qd}{\norm{\Delta \qd}} \right) \frac{\norm{\gd} \norm{\Delta \gd}}{\norm{\Delta\g}} + \frac{\Delta \qd}{\norm{\Delta \qd}} \left( \frac{\norm{\Delta \gd} \norm{\gd}}{\norm{\Delta \g}} - \frac{\norm{\Delta \qd} \norm{\qd}}{\norm{\Delta q}} \right) \right\| \\
    & \leq & \underbrace{ 2 \frac{\norm{\gd} \norm{\Delta \gd}}{\norm{\Delta \g}} \left( 1 - \cos\widehat{(\Delta \qd, \Delta \gd)} \right) }_\text{angular deviation term ($\theta$)} + \underbrace{ \left| \frac{\norm{\gd} \norm{\Delta \gd}}{\norm{\Delta \g}} - \frac{\norm{\Delta \qd} \norm{\qd}}{\norm{\Delta q}} \right| }_\text{norm deviation term (N)}
\end{eqnarray*}
The factor $\frac{2 \norm{\gd} \norm{\Delta \gd}}{\norm{\Delta \g}}$ before
$(\theta)$ is $O(1)$ when $\Delta t \to 0$, while simple vector geometry then
shows that
\begin{equation*}
    \sin\widehat{(\Delta \qd, \Delta \gd)}
    \ \leq \ \frac{\distg(x) + \distg(x')}{\norm{\Delta \gd}}
    \ \leq \ \frac{\rho}{\mdd \Delta t},
\end{equation*}
where $\mdd := \min_t \norm{\gdd(t)}$. From Lemma~\ref{lemma:minacc}, we
can assume this minimum acceleration to be strictly positive. Then, it follows
from $\rho = O(\Delta t^2)$ that the sine above is $O(\Delta t)$. Recalling
the fact that $1 - \cos\theta < \sin\theta$ for any $\theta \in [0, \pi/2]$, we
have $(\theta) = O(\Delta t)$. 

Finally,
\begin{eqnarray*}
    (\text{N}) 
    & \leq & \frac{\norm{\Delta \gd}}{\norm{\Delta \g}} \abs{\norm{\gd} - \nqd} + \nqd \abs{\frac{\norm{\Delta \gd}}{\norm{\Delta \g}} - \frac{\ndqd}{\ndq}} \\
    & \leq & O(\Delta t \cdot \rho) + \nqd \frac{(\ndq + \ndqd)\,O(\rho)}{\ndq (\ndq + O(\rho))} \\
    & \leq & O(\Delta t \cdot \rho) + \frac{\nqd \rho}{\ndq + O(\rho)}
 + \frac{\|\qd\| \ndqd}{\ndq} \frac{O(\rho)}{\ndq + O(\rho)}
    \end{eqnarray*}
Where we used the fact that $\norm{\Delta \g} \leq \distg(x) + \ndq + \distg(x')
= \ndq + O(\rho)$, and similarly for $\|\Delta \gd\|$.  Because $\| \Delta q \|
= \| \qd \| \Delta t_\disc + O(\Delta t_\disc^2)$ and $\rho / \Delta t_\disc
= O(\Delta t)$, the last two fractions are $O(\Delta t)$, so our last term
$(\text{N}) = O(\Delta t)$.

Overall, we have derived an upper bound $|u(\tau)| \leq (1 - \delta)
\taumax + O(\Delta t)$. As a consequence, there exists a constant
$\delta t > 0$ such that, whenever $\Delta t \leq \delta t$,
interpolated torques satisfy $|u| \leq \taumax$ and the interpolated
trajectory $\gint = \INTERPOLATE(x, x')$ is admissible. Note that the
constant $\delta t$ is uniform, in the sense that it does not depend on
the index $t$ on the trajectory.

\paragraph{Conclusion of the Proof} We have effectively constructed the
attraction sequence conjectured in~\cite{lavalle2001randomized}. We can now
conclude the proof similarly to the strategy sketched in that paper. Let us
denote by $\calB_t := \calB((\g, \gd)(t), \delta \rho)$, the ball of radius
$\delta \rho$ centered on $(\g, \gd)(t) \in \calX$, where $\delta \rho
= O(\delta t^2)$ as before. Suppose that the roadmap contains a state $x \in
\calB_t$, and let $t' := t + \delta t$. If the planner samples a state $x' \in
\calB_{t'}$, the interpolation between $x$ and $x'$ will be successful and $x'$
will be added to the roadmap. Since the volume of $\calB_{t'}$ is non-zero,
the event $\{ \SAMPLE(\Xfree) \in \calB_{t'}\}$ will happen with probability
one as the number of extensions goes to infinity. At the initialization of the
planner, the roadmap is reduced to $\xinit = (\g(0), \gd(0))$. Therefore,
using the property above, by induction on the number of time steps $\delta t$,
the last state $\xgoal = (\g(T), \gd(T))$ will be eventually added to the
roadmap with probability one, and the planner will find an admissible
trajectory connecting $\xinit$ to $\xgoal$.~$\blacksquare$

\section{Completeness and state-based steering in practice}
\label{sec:simulations}

Shkolnik et al.~\cite{shkolnik2009icra} showed how RRTs could not be directly
applied to kinodynamic planning due to their poor expansion rate at the
boundaries of the roadmap. They illustrated this phenomenon on the planning
problem of swinging up a (single) pendulum vertically against gravity. Let us
consider the same system, \ie the 1-DOF single pendulum depicted in
Figure~\ref{fig:pendulum} (A), with length $l = 20$ cm and mass $m = 8$ kg. It
satisfies the system assumptions of Theorem~\ref{th:completeness} \emph{a
fortiori}, as we saw that they apply to the double pendulum.

We assume that the single actuator of the pendulum, corresponding to the joint
angle $\theta$ in Figure~\ref{fig:pendulum}, has limited actuation power: $|
\tau | \leq \tau_{\max}$. The static equilibrium of the system requiring the
most torque is given at $\theta = \pm \pi / 2$ with $\tau = \frac12
l m g \approx 7.84$ Nm. Therefore, when $\tau_{\max} < 7.84$ Nm, it is
impossible for the system to raise upright directly, and the pendulum rather
needs to swing back and forth to accumulate kinetic energy before it can swing
up. For any $\tau_{\max} > 0$, the pendulum can achieve the swingup in a finite
number of swings $N$, with $N \to \infty$ as $\tau_{\max} \to 0$.

\subsection{Bezier interpolation}

A common solution~\cite{jolly2009ras, vskrjanc2010ras, hauser2013rss} to
connect two states $(q, \qd)$ and $(q', \qd')$ is the cubic Bezier curve (also
called ``Hermit curve'') which is the quadratic function $B(t)$ such that $B(0)
= q$, $\dot{B}(0) = \qd$, $B(T) = q'$ and $\dot{B}(T) = q'$, where $T$ is the
fixed duration of the interpolated trajectory. Its expression is given by:
\begin{equation*}
    B(t) \ = \ \frac{- 2 \Delta q + T (\qd + \qd')}{T^3} \, t^3 
    + \frac{3 \Delta q - 2 \qd - \qd'}{T^2} \, t^2 
    + \qd t + q
\end{equation*}
This interpolation is straightforward to implement, however it does not verify
our Assumption~\ref{ass-int-3}, as for instance
\begin{equation}
    \label{acc-vel-tie}
    \ddot{B}(0) \ = \frac{6 \Delta q - 4 \qd - 2 \qd'}{T^2} 
    \ \xrightarrow{\Delta x \to 0} \ \frac{-6 \qd}{T^2} 
    \ \neq \ 0.
\end{equation}
Our proof of completeness does not apply to such interpolators: even though
a feasible trajectory is sampled as closely as possible $(\Delta x \to 0)$, the
interpolated acceleration will \emph{not} approximate the smooth acceleration
underlying the feasible trajectory. 

\begin{prop}
    \label{prop:bezier}
    A randomized motion planner interpolating pendulum trajectories by Bezier
    curves with a fixed duration $T$ cannot find non-quasi static solutions by
    increasing sampling density.
\end{prop}

\begin{proof}
When actuation power decreases, the pendulum needs to store kinetic energy in
order to swing up, which implies that all swingup trajectories go through
velocities $|\dot{\theta}| > \dot{\theta}_{\textrm{swingup}}(\tau_{\max})$. The
function $\dot{\theta}_{\textrm{swingup}}$ increases to a positive limit
$\dot{\theta}_{\textrm{swingup}}^{\textrm{lim}}$ as $\tau_{\max} \to 0$,
where $\dot{\theta}_{\textrm{swingup}}^{\textrm{lim}} > \sqrt{8 g / l}$
from energetic considerations.\footnote{
    The expression $\dot{\theta} = \sqrt{8 g / l}$ corresponds to the kinetic
    energy $\frac14 m l \dot{\theta}^2 = m g l$, the latter being the
    (potential) energy of the system at rest in the upward equilibrium. During
    a successful last swing, the kinetic energy at $\theta = 0$ is $\frac14
    m l \dot{\theta}_{\textrm{swingup}}^2 + \calW_g + \calW_\tau = m g l$, with
    $\calW_g < 0$ the work of gravity and $\calW_\tau$ the work of actuation
    forces between $\theta=0$ and $\theta=\pi$. The work $\calW_\tau$ vanishes
    when $\tau_{\max} \to 0$.
}
Yet, feasible accelerations are also bound by $| \ddot{\theta}| \leq
K \tau_{\max}$ for some constant $K > 0$. Combining both observations in
\eqref{acc-vel-tie} yields:
\begin{equation*}
    K \tau_{\max} \geq 6 \frac{|\dot{\theta}|}{T^2} \ > \ 6 \frac{\dot{\theta}_{\textrm{swingup}}(\tau_{\max})}{T^2} 
    \quad \Rightarrow \quad 
    \dot{\theta}_{\textrm{swingup}}(\tau_{\max}) \leq \frac{KT^2}{6} \tau_{\max}.
\end{equation*}
Since the planner uses a constant $T$ and $\dot{\theta}_{\textrm{swingup}}$
increases to $\dot{\theta}_{\textrm{swingup}}^{\textrm{lim}} > \sqrt{8 g / l}$
when $\tau_{\max}$ decreases to 0, this inequality cannot be satisfied for
arbitrary small actuation power $\tau_{\max}$. Hence, even with an arbitrarily
high sampling density around a feasible trajectory $\g(t)$, the planner will
not be able to reconstruct a feasible approximation $\gint(t)$.
\end{proof}

\subsection{Second-order continuous interpolation}

Let $\qd_\avg := \frac12 (\qd + \qd')$ denote the average velocity between $(q,
\qd)$ and $(q', \qd')$. Since the system has only one degree of freedom, one
can interpolate trajectories that comply with our Assumption~\ref{ass-int-3}
using constant accelerations with a suitable trajectory duration:
\begin{equation*}
    \begin{array}{rcl}
        C : [0, \Delta t_C] & \to & ]-\pi, \pi] \\
        t & \mapsto & C(t) = q + t \qd + \frac{t^2}{2} (\Delta \qd / \Delta t_C).
    \end{array}
\end{equation*}
One can check that choosing $\Delta t_C = (\Delta q / \qd_\avg)$ results in
$\dot{C}(0)=\qd$, $\dot{C}(\Delta t_C)=\qd'$, $C(0) = q$ and $C(\Delta t_C)=q'$. This
duration is similar to the term $\Delta t_\disc$ in Assumption~\ref{ass-int-3},
with both expressions converging to the same value as $\Delta x \to 0$. We call $C(t)$
the \emph{second-order continuous 1-DOF} (SOC1) interpolation.

Note that this interpolation function only applies to single-DOF systems. For
multi-DOF systems, the correct duration $\Delta t_C$ used to transfer from one state
to another is different for each DOF, hence constant accelerations cannot be
used. One can then apply optimization techniques~\cite{perez2012lqr,
pham2013kinodynamic} or use a richer family of curves such as piecewise
linear-quadratic segments~\cite{hauser2010icra}.

\subsection{Comparison in simulations}

According to Theorem~\ref{th:completeness} and our previous discussion,
a randomized planner based on Bezier interpolation is not expected to be
probabilistically complete as $\tau_{\max} \to 0$, while the same planner using
the SOC1 interpolation will be complete at any rate. We asserted this statement
in simulations of the pendulum with RRT~\cite{lavalle2000rapidly}.

Our implementation of RRT is that described in Algorithm~\ref{algo:planner},
with the addition of the \emph{steer-to-goal} heuristic: every $m=100$ steps,
the planner tries to steer to $\xgoal$ rather than $\xrand$. This extra step
speeds up convergence when the system reaches the vicinity of the goal area.
We use uniform random sampling for $\SAMPLE(S)$, while for $\PARENTS(x',V)$
returns the $k=10$ nearest neighbors of $x'$ in the roadmap $V$. All the source
code used in these experiments can be accessed at~\cite{code}.

\begin{figure}[p]
  \centering
  \includegraphics[height=7cm]{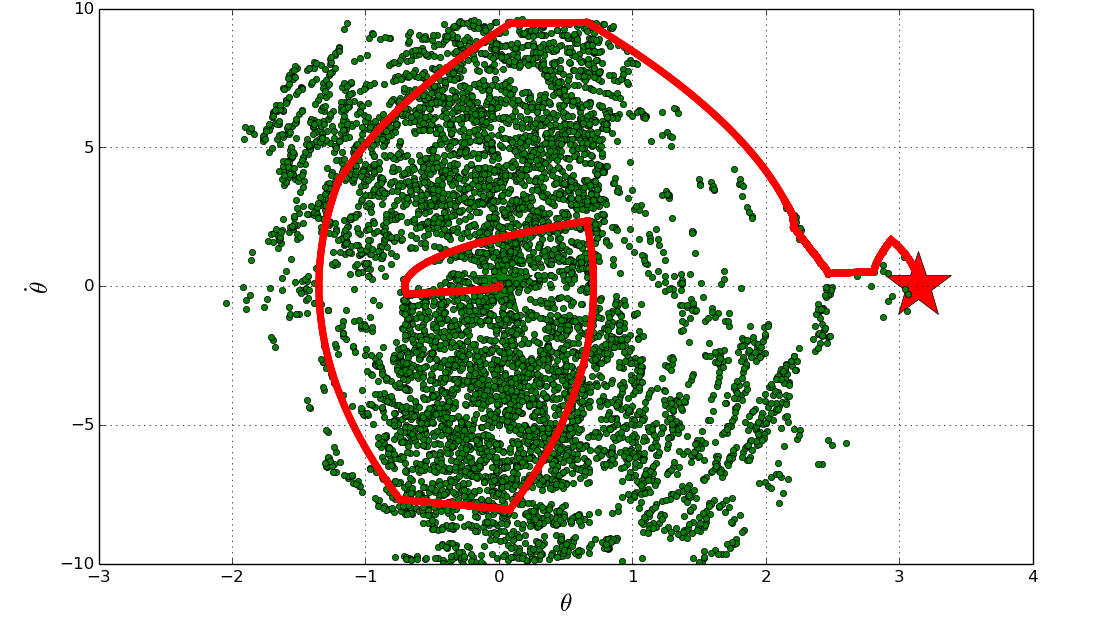}
  \caption{
      Phase-space portrait of the roadmap constructed by RRT using the
      second-order continuous (SOC1) interpolation. The planner found
      a successful trajectory (red line) after 26,300 extensions. This planner
      is probabilistically complete (Theorem~\ref{th:completeness}) thanks to
      the fact that SOC1 curves satisfy Assumption~\ref{ass-int-3}.
  }
  \label{fig:rrt-dac1}
\end{figure}
\begin{figure}[p]
  \centering
  \includegraphics[height=7cm]{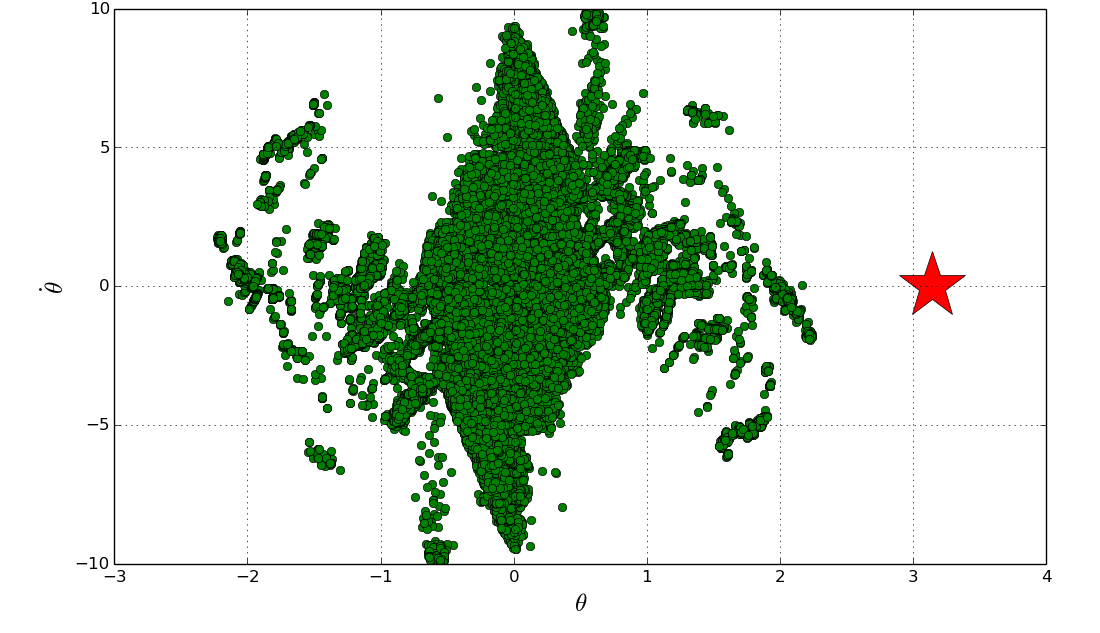}
  \caption{
      Roadmap constructed by RRT after 100,000 extensions using the Bezier
      interpolation. Reachable states are distributed in two major areas:
      a central, diamond shape corresponding to the states that the planner can
      connect at any rate, and two cones directed towards the goal ($\theta
      = \pi$ or $\theta=-\pi$). Even after several days of computations, this
      planner could not find a successful motion plan. Our completeness
      theorem does not apply to this planner because Bezier curves do not
      satisfy Assumption~\ref{ass-int-3}.
  }
  \label{fig:rrt-bezier}
\end{figure}

We compared the performance of RRT with the Bezier and SOC1 interpolations, all
other parameters being the same, on a single pendulum with $\tau_{\max}
= 5$~Nm. The RRT-SOC1 combo found a four-swing solution after 26,300 RRT
extensions, building a roadmap with 6434 nodes (Figure~\ref{fig:rrt-dac1}).

Meanwhile, even after one day of computations and more than 200,000 RRT
extensions, the RRT-Bezier combo could not find any solution.
Figure~\ref{fig:rrt-bezier} shows the roadmap at 100,000 extensions (26,663
nodes). Interestingly, we can distinguish two zones in this roadmap. The first
one is a dense, diamond-shape area near the downward equilibrium $\theta=0$. It
corresponds to states that are straightforward to connect by Bezier
interpolation, and as expected from Proposition~\ref{prop:bezier}, velocities
$\dot{\theta}$ in this area decrease sharply with $\theta$. The second one
consists of two cones directed torwards the goal. Both areas exhibit a higher
density near the axis $\dot{\theta} = 0$, which is also consistent with
Proposition~\ref{prop:bezier}.

The comparison of the two roadmaps is clear: with a second-order continuous
interpolation, the RRT-SOC1 planner leverages additional sampling into
exploration of the state space. Conversely, RRT-Bezier lacks this property
(Proposition~\ref{prop:bezier}), and its roadmap stays confined to a subset of
the pendulum's reachable space.

\section{Conclusion}

In this paper, we provided the first ``operational'' proof of probabilistic
completeness for a large class of randomized kinodynamic planners, namely those
that interpolate state-space trajectories. We observed that an important
ingredient for completeness is the ``continuity'' of the interpolation
procedure, which we characterized by the \emph{second-order continuity} (SOC)
property. In particular, we found in simulation experiments that this property
is critical to planner performances: a standard RRT with second-order
continuous interpolation has no difficulty finding swingup trajectories for
a low-torque pendulum, while the same RRT with Bezier interpolation (which are
not SOC) could not find any solution.  This experimentally confirms our
completeness theorem and suggests that second-order continuity is an important
design guideline for kinodynamic planners with state-based steering.

\subsection*{References}

\bibliographystyle{elsarticle-num}
\bibliography{refs}

\begin{thebibliography}{10}
\expandafter\ifx\csname url\endcsname\relax
  \def\url#1{\texttt{#1}}\fi
\expandafter\ifx\csname urlprefix\endcsname\relax\def\urlprefix{URL }\fi
\expandafter\ifx\csname href\endcsname\relax
  \def\href#1#2{#2} \def\path#1{#1}\fi

\bibitem{caron2014icra}
S.~Caron, Q.-C. Pham, Y.~Nakamura, Completeness of randomized kinodynamic
  planners with state-based steering, in: Robotics and Automation (ICRA), 2014
  IEEE International Conference on, IEEE, 2014, pp. 5818--5823.

\bibitem{latombe1991book}
J.-C. Latombe, Robot motion planning, Vol. 124, Springer US, 1991.

\bibitem{Lav06book}
S.~LaValle, Planning algorithms, Cambridge Univ Press, 2006.

\bibitem{kavraki1996probabilistic}
L.~E. Kavraki, P.~Svestka, J.-C. Latombe, M.~H. Overmars, Probabilistic
  roadmaps for path planning in high-dimensional configuration spaces, Robotics
  and Automation, IEEE Transactions on 12~(4) (1996) 566--580.

\bibitem{hsu2002randomized}
D.~Hsu, R.~Kindel, J.-C. Latombe, S.~Rock, Randomized kinodynamic motion
  planning with moving obstacles, The International Journal of Robotics
  Research 21~(3) (2002) 233--255.

\bibitem{lavalle2001randomized}
S.~M. LaValle, J.~J. Kuffner, Randomized kinodynamic planning, The
  International Journal of Robotics Research 20~(5) (2001) 378--400.

\bibitem{karaman2011sampling}
S.~Karaman, E.~Frazzoli, Sampling-based algorithms for optimal motion planning,
  The International Journal of Robotics Research 30~(7) (2011) 846--894.

\bibitem{hsu1997path}
D.~Hsu, J.-C. Latombe, R.~Motwani, Path planning in expansive configuration
  spaces, in: Robotics and Automation, 1997. Proceedings., 1997 IEEE
  International Conference on, Vol.~3, IEEE, 1997, pp. 2719--2726.

\bibitem{lozano1983spatial}
T.~Lozano-Perez, Spatial planning: A configuration space approach, Computers,
  IEEE Transactions on 100~(2) (1983) 108--120.

\bibitem{donald1993kinodynamic}
B.~Donald, P.~Xavier, J.~Canny, J.~Reif, Kinodynamic motion planning, Journal
  of the ACM (JACM) 40~(5) (1993) 1048--1066.

\bibitem{Lau98book}
J.-P. Laumond, Robot Motion Planning and Control, Springer-Verlag, New York,
  1998.

\bibitem{bullo2001kinematic}
F.~Bullo, K.~M. Lynch, Kinematic controllability for decoupled trajectory
  planning in underactuated mechanical systems, Robotics and Automation, IEEE
  Transactions on 17~(4) (2001) 402--412.

\bibitem{BobX85ijrr}
J.~Bobrow, S.~Dubowsky, J.~Gibson, Time-optimal control of robotic manipulators
  along specified paths, The International Journal of Robotics Research 4~(3)
  (1985) 3--17.

\bibitem{wieber2002}
P.-B. Wieber, On the stability of walking systems, in: Proceedings of the
  international workshop on humanoid and human friendly robotics, 2002.

\bibitem{caron2015rss}
S.~Caron, Q.-C. Pham, Y.~Nakamura, Leveraging cone double description for
  multi-contact stability of humanoids with applications to statics and
  dynamics, in: Robotics: Science and System, 2015.

\bibitem{bialkowski2013iros}
J.~Bialkowski, M.~Otte, E.~Frazzoli, Free-configuration biased sampling for
  motion planning, in: Intelligent Robots and Systems (IROS), 2013 IEEE/RSJ
  International Conference on, IEEE, 2013, pp. 1272--1279.

\bibitem{shkolnik2009icra}
A.~Shkolnik, M.~Walter, R.~Tedrake, Reachability-guided sampling for planning
  under differential constraints, in: Robotics and Automation, 2009. ICRA'09.
  IEEE International Conference on, IEEE, 2009, pp. 2859--2865.

\bibitem{papadopoulos2014arxiv}
G.~Papadopoulos, H.~Kurniawati, N.~M. Patrikalakis, Analysis of asymptotically
  optimal sampling-based motion planning algorithms for lipschitz continuous
  dynamical systems, arXiv preprint arXiv:1405.2872.

\bibitem{kunz2014icra}
T.~Kunz, M.~Stilman, Probabilistically complete kinodynamic planning for robot
  manipulators with acceleration limits, in: Intelligent Robots and Systems
  (IROS 2014), 2014 IEEE/RSJ International Conference on, IEEE, 2014, pp.
  3713--3719.

\bibitem{perez2012lqr}
A.~Perez, R.~Platt, G.~Konidaris, L.~Kaelbling, T.~Lozano-Perez, Lqr-rrt*:
  Optimal sampling-based motion planning with automatically derived extension
  heuristics, in: Robotics and Automation (ICRA), 2012 IEEE International
  Conference on, IEEE, 2012, pp. 2537--2542.

\bibitem{tedrake2009lqr}
R.~Tedrake, Lqr-trees: Feedback motion planning on sparse randomized trees.

\bibitem{kunz2015afr}
T.~Kunz, M.~Stilman, Kinodynamic rrts with fixed time step and best-input
  extension are not probabilistically complete, in: Algorithmic Foundations of
  Robotics XI, Springer, 2015, pp. 233--244.

\bibitem{karaman2010optimal}
S.~Karaman, E.~Frazzoli, Optimal kinodynamic motion planning using incremental
  sampling-based methods, in: Decision and Control (CDC), 2010 49th IEEE
  Conference on, IEEE, 2010, pp. 7681--7687.

\bibitem{karaman2013sampling}
S.~Karaman, E.~Frazzoli, Sampling-based optimal motion planning for
  non-holonomic dynamical systems, in: {IEEE} Conference on Robotics and
  Automation (ICRA), 2013.

\bibitem{jolly2009ras}
K.~Jolly, R.~S. Kumar, R.~Vijayakumar, A bezier curve based path planning in a
  multi-agent robot soccer system without violating the acceleration limits,
  Robotics and Autonomous Systems 57~(1) (2009) 23--33.

\bibitem{vskrjanc2010ras}
I.~{\v{S}}krjanc, G.~Klan{\v{c}}ar, Optimal cooperative collision avoidance
  between multiple robots based on bernstein--b{\'e}zier curves, Robotics and
  Autonomous systems 58~(1) (2010) 1--9.

\bibitem{hauser2013rss}
K.~Hauser, Fast interpolation and time-optimization on implicit contact
  submanifolds., in: Robotics: Science and Systems, Citeseer, 2013.

\bibitem{pham2013kinodynamic}
Q.-C. Pham, S.~Caron, Y.~Nakamura, Kinodynamic planning in the configuration
  space via velocity interval propagation, Robotics: Science and System.

\bibitem{hauser2010icra}
K.~Hauser, V.~Ng-Thow-Hing, Fast smoothing of manipulator trajectories using
  optimal bounded-acceleration shortcuts, in: Robotics and Automation (ICRA),
  2010 IEEE International Conference on, IEEE, 2010, pp. 2493--2498.

\bibitem{lavalle2000rapidly}
S.~M. LaValle, J.~J. Kuffner~Jr, Rapidly-exploring random trees: Progress and
  prospects.

\bibitem{code}
Source code to be published online.

\end{thebibliography}

\newpage
\appendix

\section{Proofs of the lemmas}
\label{appendix:lemmas}

\setcounter{lemma}{0}  

\begin{lemma}
Let $g : [0, \Delta t] \to \RR^k$ denote a smooth Lipschitz function. Then, for any
$(t, t') \in [0, \Delta t]^2$,
    \begin{equation*}
    \ \left\| \dot{g}(t) - \frac{g(t') - g(t)}{|t' - t|} \right\| 
        \ \leq \ \frac{K_g}{2} | t' - t |.
    \end{equation*}
\end{lemma}

\begin{proof}
For $t' > t$,
    \begin{eqnarray*}
    \left\| \dot{g}(t) - \frac{g(t') - g(t)}{t' - t} \right\|
    & \leq & \frac{1}{t' - t} \left \| 
        \int_t^{t'} (\dot{g}(t) - \dot{g}(w)) \d{w}
        \right\| \\
    & \leq & \frac{1}{t' - t} \int_t^{t'} \norm{\dot{g}(t) - \dot{g}(w)} \d{w} \\
    & \leq & \frac{K_g}{t' - t} \int_t^{t'} | t - w | \d{w} \\
    & \leq & \frac{K_g}{2} (t' - t). \qedhere
    \end{eqnarray*}
\end{proof}

\begin{lemma}
    If there exists an admissible trajectory $\g$ with $\delta$-clearance in
    control space, then there exists $\delta' < \delta$ and a neighboring
    admissible trajectory $\g'$ with $\delta'$-clearance in control space which
    is always accelerating, \ie such that $\| \gdd' \|$ is always greater than
    some constant $\mdd > 0$.
\end{lemma}

\begin{proof}
    If there is a time interval $[t, t']$ on which $\gdd \equiv 0$, suffices to
    add a wavelet function $\delta \gdd_i$ of arbitrary small amplitude $\delta
    \qdd_i$ and zero integral over $[t, t']$ to generate a new trajectory $\gdd
    + \delta \gdd$ where the acceleration cancels on at most a discrete number
    of time instants. Adding accelerations $\delta \gdd_i$ directly is possible
    thanks to full actuation, while $\delta'$-clearance can be achieved for
    $\delta' \leq \delta$ by taking sufficiently small amplitudes $\delta
    \qdd_i$.

    Suppose now that the roots of $\gdd$ form a discrete set $\{ t_0, t_1,
    \ldots, t_m \}$. Let $t_0$ be one of these roots, and let $[t, t']$ denote
    a neighbordhood of $t_0$. Repeat the process of adding wavelet functions
    $\delta \gdd_i$ and $\delta \gdd_j$ of zero integral over $[t, t']$ and
    arbitrary small amplitude to two coordinates $i$ and $j$, but this time
    enforcing that the sum of the two wavelets satisfies $| \delta \gdd_i
    + \delta \gdd_j| \geq \epsilon_{ij} > 0$. This method ensures that the root
    $t_0$ is eliminated (either $\gdd_i(t_0) \neq 0$ or $\gdd_j(t_0) \neq 0$)
    without introducing new roots. We conclude by iterating the process on the
    finite set of roots.
\end{proof}

\end{document}